\documentclass[10pt,twocolumn,letterpaper]{article}

\usepackage[pagenumbers]{cvpr} 

\usepackage{graphicx}
\usepackage{amsmath, amsthm}
\usepackage{amssymb}
\usepackage{booktabs}

\usepackage{graphicx}
\usepackage{makecell}
\usepackage{color}
\usepackage{enumitem}
\usepackage{multirow}
\usepackage{booktabs}
\usepackage{bbm}
\usepackage{arydshln}
\usepackage{subcaption}
\usepackage{bm}
\usepackage[accsupp]{axessibility}

\newtheorem{proposition}{Proposition}

\newcommand{\bb}[1]{\mathbf{#1}}

\newcommand{\bxi}{\bb{x}_\text{img}}
\newcommand{\bxt}{\bb{x}_\text{txt}}

\newcommand{\bzth}{\hat{\mathbf{z}}_\text{txt}}
\newcommand{\bzthi}{\hat{\mathbf{z}}^i_\text{txt}}
\newcommand{\bzthj}{\hat{\mathbf{z}}^j_\text{txt}}
\newcommand{\bzthk}{\hat{\mathbf{z}}^k_\text{txt}}
\newcommand{\tbzth}{\hat{\mathbf{z}}^T_\text{txt}}
\newcommand{\bzih}{\hat{\mathbf{z}}_\text{img}}
\newcommand{\bep}{\bm{\epsilon}}

\newcommand{\bz}{\bb{z}}
\newcommand{\bzero}{\bb{0}}
\newcommand{\bI}{\bb{I}}
\newcommand{\bzt}{\bb{z_\text{txt}}}
\newcommand{\bzi}{\bb{z_\text{img}}}
\newcommand{\bzii}{\bb{z}^i_\text{img}}
\newcommand{\bzij}{\bb{z}^j_\text{img}}
\newcommand{\bzik}{\bb{z}^k_\text{img}}
\newcommand{\tbzi}{\bb{z}^T_\text{img}} 
\newcommand{\bzpt}{\bb{z'_\text{txt}}}

\newcommand{\bgi}{\bb{g}_\text{img}}

\usepackage[pagebackref,breaklinks,colorlinks]{hyperref}

\usepackage[capitalize]{cleveref}
\crefname{section}{Sec.}{Secs.}
\Crefname{section}{Section}{Sections}
\Crefname{table}{Table}{Tables}
\crefname{table}{Tab.}{Tabs.}

\def\cvprPaperID{5522} 
\def\confName{CVPR}
\def\confYear{2023}

\begin{document}

%
\title{Variational Distribution Learning for Unsupervised Text-to-Image Generation} 
\author{
Minsoo Kang\textsuperscript{\normalfont 1}\thanks{This work was partly done during an internship at Kakao Brain.} \qquad Doyup Lee\textsuperscript{\normalfont 3} \qquad
Jiseob Kim\textsuperscript{\normalfont 3} \qquad Saehoon Kim\textsuperscript{\normalfont 3} \qquad Bohyung Han\textsuperscript{\normalfont 1,2} \\
{\hspace{-0.8cm} \textsuperscript{1}ECE \& \textsuperscript{2}IPAI, Seoul National University}~~~~~~\textsuperscript{3}Kakao Brain \\
 {\tt\small \{kminsoo, bhhan\}@snu.ac.kr \quad \{doyup.lee, jiseob.kim, shkim\}@kakaobrain.com}
}
\maketitle


\begin{abstract}
We propose a text-to-image generation algorithm based on deep neural networks when text captions for images are unavailable during training. 
In this work, instead of simply generating pseudo-ground-truth sentences of training images using existing image captioning methods, we employ a pretrained CLIP model, which is capable of properly aligning embeddings of images and corresponding texts in a joint space and, consequently, works well on zero-shot recognition tasks.    
We optimize a text-to-image generation model by maximizing the data log-likelihood conditioned on pairs of image-text CLIP embeddings. 
To better align data in the two domains, we employ a principled way based on a variational inference, which efficiently estimates an approximate posterior of the hidden text embedding given an image and its CLIP feature.
Experimental results validate that the proposed framework outperforms existing approaches by large margins under unsupervised and semi-supervised text-to-image generation settings.
\end{abstract}


\section{Introduction}
\label{sec:introduction}
Recent advances in text-to-image (T2I) generation techniques~\cite{dalle, glide, cogview, VQGAN, rqvae, LDM, dalle2, imagen, parti, draftNrevise} have shown promising results by employing generative adversarial networks~\cite{goodfellow2014generative}, autoregressive models~\cite{van2016pixel}, or diffusion models~\cite{ddpm, songscore} to synthesize images based on their text captions.
However, these approaches require a paired dataset that consists of images and their corresponding text captions, and, consequently, incur significant annotation costs, especially for labeling image captions. 
To alleviate this limitation, unsupervised learning methods for T2I generation have recently drawn attention to the computer vision community, where the models learn to generate images without paired text captions.

Existing T2I models~\cite{wang2022clip, zhou2022towards, kNNDiff, rdm} based on unsupervised learning exploit Contrastive Language-Image Pretraining (CLIP)~\cite{radford2021learning} to sidestep the absence of text captions during training.
Specifically, after a text embedding is estimated using a given image embedding, the T2I model is trained to synthesize an image conditioned on the estimated text embedding.
However, although image and text embeddings extracted by CLIP are not accurately aligned, existing approaches assume that the distinction is ignorable~\cite{wang2022clip} or simple to recover by just adding Gaussian noises~\cite{zhou2022towards} without considering the underlying structure of text embeddings.
Thus, those algorithms may suffer from large discrepancies between true and estimated text embeddings at both training and testing.
To tackle the challenge, we propose a variational distribution learning technique for unsupervised T2I generation, where the lower-bound of the data log-likelihood is maximized in a principled way.
Specifically, we first regard a text embedding as a hidden random variable while an image and its CLIP embedding are observable random variables.
Then, we decompose the variational lower-bound into three parts: 1) the similarity between the text embedding prior and posterior, 2) the log-likelihood of the image embedding given the text embedding, 3) the log-likelihood of the image given the image and text embeddings in the trained T2I model.
Since the lower-bound formulation enforces the matching between the prior and posterior distributions of text embedding, our method achieves a more accurate estimation of the embedding and reduces the discrepancy between the true and estimated embeddings.

For the optimization of the variational objective, we employ a two-stage training strategy for T2I models.
In the first stage, we learn an encode-decoder architecture that takes the image embedding as an input and the estimated text embedding as a latent bottleneck.
Then, our network estimates two conditional distributions of CLIP embeddings, one for the variational distribution of the text embedding given the image embedding and the other for the model distribution of the image embedding given the text embedding.
The parameters of the two distributions are obtained from the first two terms in the variational lower-bound objective.
Note that we relax the Kullback-Leibler (KL) divergence term in the training objective of the first stage to an adversarial training loss, specifically, the Jensen-Shannon divergence. 
Since the KL divergence is only tractable for a confined family of distributions, this relaxation allows more flexible choices for the conditional and the prior distributions.
In the second stage, a T2I model learns the conditional distribution of the images given the estimated text embeddings and the image features. 
Altogether, the proposed method achieves outstanding performance on widely adopted datasets~\cite{lin2014microsoft, cc3m}.
The main contributions of our work are summarized below: 
\begin{itemize}
 \item We propose a principled approach for unsupervised and semi-supervised text-to-image generation tasks based on a variational inference technique.
 \item We theoretically show that our method considers the underlying structure of text embeddings, which can eventually lead to better generalization performance. 
 \item We empirically confirm that the proposed algorithm outperforms the existing methods by large margins.
\end{itemize}

The rest of our paper is organized as follows. 
Section~\ref{sec:related} overviews the related work about unsupervised training methods in text-to-image generation. 
Section~\ref{sec:proposed algorithm} describes the main idea of our approach while Sections~\ref{sec:two_stage_framework} and~\ref{sub:training_t2i} discuss the procedures of the first and second training stages, respectively.
The experimental results are presented in Section~\ref{sec:experiments}, and we finally conclude this paper in Section~\ref{sec:conclusion}.

\section{Related Work}
\label{sec:related}
Text-to-image generative models have shown astonishing performance via learning with large-scale datasets composed of image-text pairs.
Existing algorithms often represent each image with a sequence of discrete tokens and learn autoregressive~\cite{dalle,VQGAN,rqvae,cogview,nuwa} or bidirectional~\cite{draftNrevise} transformers to generate high-resolution images given text inputs.
Recently, the introduction of diffusion models~\cite{ddpm,cascadedDiff,adm,clsfree} has paved the way to learn large-scale T2I models and generate high-quality images conditioned on text.

For training T2I models without captions describing images, previous approaches~\cite{wang2022clip,zhou2022towards,kNNDiff} typically exploit the pretrained CLIP~\cite{radford2021learning} to approximate missing text captions.
Specifically, CLIP-GEN~\cite{wang2022clip} assumes that a CLIP image embedding is perfectly aligned with the corresponding text embedding, and utilizes the image embedding as a proxy of its text embedding for T2I generation.
On the other hand, LAFITE~\cite{zhou2022towards} adds a Gaussian random noise to the image embedding for estimating the unknown true text embedding.
However, these algorithms fail to consider the underlying structure of the text embeddings, which eventually results in the imprecise approximation of text embedding.
Retrieval-based approaches~\cite{kNNDiff,rdm} employ image features similar to CLIP-GEN~\cite{wang2022clip} for text conditions while $k$-nearest image embeddings are additionally utilized for the construction of the conditions; note that these approaches are orthogonal to our method and can be combined with the proposed method to further enhance generation performance.

On the other hand, we propose a principled framework relying on a variational inference to effectively reduce the discrepancy between the true embedding employed during inference and the approximated one drawn by the variational distribution used for training.

%
\section{Main Framework}
\label{sec:proposed algorithm}
This section presents the existing unsupervised training approaches for T2I models.
Then, we formulate our variational inference framework for unsupervised training that effectively reduces the discrepancy between true and estimated text embeddings.
Figure~\ref{fig:graphical_model} illustrates the graphical model of the data generating process in our approach.
\begin{figure}[t]
\centering
\includegraphics[width=0.96\linewidth]{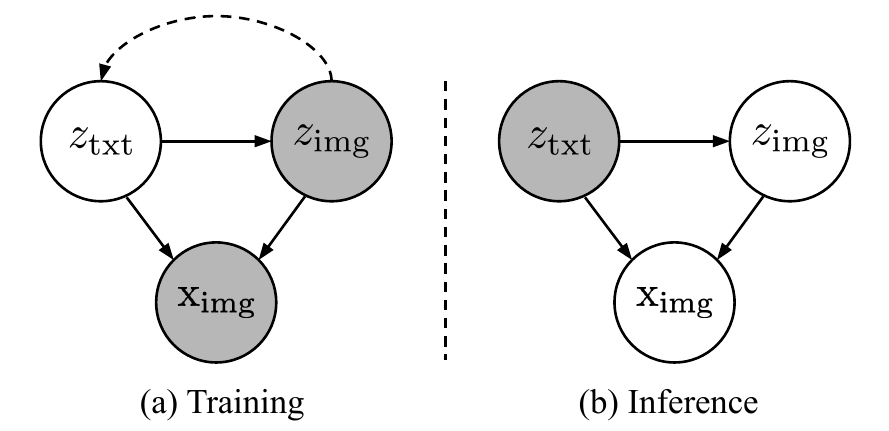}
\caption{Probabilistic graphical models of the proposed method, where shaded nodes represent observable random variables while unshaded nodes represent hidden random variables.}
\label{fig:graphical_model}
\end{figure}

\subsection{Unsupervised Training of T2I Models}
Unsupervised training of T2I models learns to generate images without textual annotations while the model generates high-quality images based on given captions at inference time.
Conventional supervised training of T2I models exploits both an image, $\bxi$, and its textual caption, $\bxt$, to estimate a conditional distribution $p(\bxi|\bxt)$.
However, unsupervised training assumes that the text condition $\bxt$ is unavailable during training.
Existing approaches estimate a latent text representation $\bzt$ and formulate the T2I model as a task to derive the conditional distribution $p(\bxi | \bzt)$.
In other words, an encoder $E(\cdot)$ approximates the text condition, \ie, $\bzth=E(\bxi)$, during training and $\bzth$ is replaced with a representation of a given text sentence $\bxt$ at inference time.

The previous studies commonly incorporate the pretrained CLIP model~\cite{radford2021learning}, which consists of two separated encoders for images $f_\text{img}(\cdot)$ and texts $f_\text{txt}(\cdot)$. 
The vision-language model is learned to make a pair of image and text embeddings, denoted respectively by $f_\text{img}(\bxi)$ and $f_\text{txt}(\bxt)$, have a high cosine similarity.
CLIP-GEN~\cite{wang2022clip} and retrieval-based models~\cite{kNNDiff,rdm} approximate the text condition using the CLIP image embedding, $\bzth \approx f_\text{img}(\bxi)$, during training.
Since the two embeddings are not exactly aligned, LAFITE~\cite{zhou2022towards} additionally adds a perturbation to the CLIP image embedding and approximates the text condition $\bzt$ as
\begin{align}
\bzth \approx \text{Normalize} ( f_\text{img}(\bxi) + \xi \| f_\text{img}(\bxi) \| \bep / \| \bep \| ),
\label{eq:LAFITE}
\end{align} 
where $\| \cdot \|$ and $\text{Normalize}(\cdot)$ are the Euclidean norm and an operator to divide the input by its magnitude while $\xi$ is a hyperparameter and $\bep$ is a random noise drawn from the Gaussian distribution, $\mathcal{N}(\bzero, \bI)$.
Although the previous methods are proposed to approximate the absent text condition, we observe that they are insufficient to reduce the gap between the training and the inference environments of T2I models. 

\subsection{Variational Inference for Training T2I Models}
We aim to formulate a variational inference framework for unsupervised training of T2I models to reduce the discrepancy between training and inference while improving the performance of T2I models.
Given CLIP image and text embeddings, $\bzi=f_\text{img}(\bxi)$ and $\bzt=f_\text{txt}(\bxt)$, our T2I generation model is defined as 
\begin{equation} \label{eq:t2i}
    p_{\theta^{\bxi}}(\bxi | \bzi, \bzt),
\end{equation}
where $\theta^{\bxi}$ is a set of parameters in the T2I model.
Note that the text representation $\bzt$ is unavailable during training while the image embedding $\bzi$ is absent for inference as illustrated in Figure~\ref{fig:graphical_model}.
Thus, the challenges lie in the precise approximation of the text embedding $\bzth$ based on $\bzi$ during training while approximating $\bzih$ given $\bzt$ at inference.

Our approach naturally maximizes the marginal log-likelihood $\log p_\theta(\bxi)$ with respect to $\theta$, a set of parameters of our data generating process, without the observation of the text embedding $\bzt$ during training.
The log-likelihood $\log p_\theta(\bxi)$ is computed by marginalizing out $\bzt$ as
\begin{align}
 \log p_\theta(\bxi) &= \log p_\theta(\bxi, \bzi)  \nonumber \\
 &= \log \int p_\theta(\bxi, \bzi, \bzt) d \bzt. 
\label{eq:mar}
\end{align} 
Unfortunately, it is intractable to compute the marginal log-likelihood either directly or through the estimation of the posterior, $p_\theta(\bzt | \bxi, \bzi)$.
Hence, we employ a variational inference technique that approximates the true posterior $p_\theta(\bzt |\bxi, \bzi)$ using a variational distribution $q_{\phi^{\bzt}}(\bzt | \bzi)$ parametrized by $\phi^{\bzt}$ under the assumption of $q_{\phi^\bzt}(\bzt | \bxi, \bzi)=q_{\phi^\bzt}(\bzt | \bzi)$.

To learn the true posterior distribution via the variational posterior, we equivalently maximize the lower bound on the log-likelihood, which is given by
\begin{align} 
\label{eq:vi} 
 &\log p_\theta(\bxi, \bzi) \nonumber\\
 &\geq \mathbb{E}_{q_{\phi^{\bzt}}( \bzt | \bzi)} \left[ \log \frac{p_\theta(\bxi, \bzi, \bzt)}{q_{\phi^{\bzt}}(\bzt | \bzi)} \right],
\end{align}
where the lower bound in~\eqref{eq:vi} is defined as $\mathcal{L}_\text{ELBO}$.
We achieve the tight lower bound when the variational distribution is exactly same as the true posterior.

\paragraph{ELBO}
By factorizing $p_\theta(\bxi, \bzi, \bzt)$ into the product of $p_{\theta^{\bxi}}(\bxi | \bzi, \bzt)$,  $p_{\theta^{\bzi}}(\bzi | \bzt)$, and $p(\bzt)$ according to the probabilistic graphical model in Figure~\ref{fig:graphical_model}, $\mathcal{L}_\text{ELBO}$ is decomposed as follows:
%
%
\begin{align} \nonumber
\mathcal{L}_\text{ELBO} &= -D_\text{KL} (q_{\phi^{\bzt}}( \bzt | \bzi) || p(\bzt) ) \\ \nonumber
&+ \mathbb{E}_{q_{\phi^{\bzt}}( \bzt | \bzi)}[\log p_{\theta^\bzi} (\bzi | \bzt) ] \\ 
&+ \mathbb{E}_{q_{\phi^{\bzt}}( \bzt | \bzi)} [ \log p_{\theta^{\bxi}}(\bxi | \bzi, \bzt)]. \label{eq:vi_detail}
\end{align}
We reformulate our task as the maximization of $\mathcal{L}_\text{ELBO}$, which involves the optimization of the T2I model in \eqref{eq:t2i} under the unsupervised setting.
Specifically, the maximization of the first term in \eqref{eq:vi_detail} encourages the estimated sample $\bzth$ drawn from $q_{\phi^{\bzt}}( \bzt | \bzi)$ to lie on the structure of $\bzt$ by minimizing the KL divergence.
Also, the second term learns $p_{\theta^\bzi}(\bzi | \bzt)$ to reconstruct $\bzi$ based on the given $\bzt$ from $q_{\phi^{\bzt}}( \bzt | \bzi)$, where $p_{\theta^\bzi}(\bzi | \bzt)$ is employed to estimate $\bzi$ in \eqref{eq:t2i} at inference.
Finally, the third term implies training T2I models without text captions, where $q_{\phi^\bzt}(\bzt | \bzi)$ estimates the absent text embedding $\bzt$ based on $\bzi$.
The details about how to optimize the model parameters are discussed in the next section.

To maximize $\mathcal{L}_\text{ELBO}$ of \eqref{eq:vi_detail} in practice, we employ the following two-stage optimization procedure:
\begin{enumerate}
\item \noindent{\bf Fix $\{ \theta^{\bxi} \} $ and optimize $ \theta^\bzi$ and $\phi^{\bzt}$.}

With the parameter $\{ \theta^{\bxi} \} $ fixed, maximize the objective in \eqref{eq:vi_detail} with respect to $ \theta^\bzi $ and $\phi^{\bzt}$.

\item \noindent{\bf Fix $\{ \theta^\bzi, \phi^{\bzt} \}$ and optimize $\theta^{\bxi}$.}

With the parameters $\{ \theta^\bzi, \phi^{\bzt} \}$ fixed, maximize the objective in \eqref{eq:vi_detail} with respect to $\theta^{\bxi}$.

\end{enumerate} 
In other words, we first train the generative and variational parameters to approximate $\bzi$ and $\bzt$, and then train a T2I model under the unsupervised setting.
We present the detailed description of the first and second stages in Section~\ref{sec:two_stage_framework} and~\ref{sub:training_t2i}, respectively.


\section{Variational Distribution Learning}
\label{sec:two_stage_framework}
This section describes the first stage of our algorithm based on variational distribution learning (VDL). 
VDL approximates the posterior of text features given image embeddings, $q_{\phi^{\bzt}}(\bzt | \bzi)$, via an adversarial training and reconstructs the image embeddings given the text features using $p_{\theta^\bzi} (\bzi | \bzt)$.

\subsection{Sampling}
\label{sub:sampling}
We first discuss the sampling procedure in VDL that involves the sampling of $\bzth$ for training to estimate the text embeddings and the sampling of $\bzih$ for inference to obtain the image embeddings.
Also, we describe the sampling process of the text prior $\bzpt$, which is only required for training.%

\subsubsection{Generating Samples for Text Embedding}
\label{subsub:variational_distribution_samples}
Let $G(\cdot)$ with parameters of $\phi^\bzt$ be an encoder implemented with a multilayer perceptron, 
Using the encoder, we draw a sample $\bzth$ from an implicitly defined variational distribution $q_{\phi^{\bzt}}(\bzt | \bzi)$ as follows:
\begin{align}
\bzth &\sim S_\text{VDL}(z_\text{img}, G, r)  \nonumber \\
 &:= \text{Normalize} \left( \bzi + r \cdot \frac{ G(\bzi)}{\| G(\bzi) \|} \right), 
\label{eq:fake_txt}
\end{align}
where $S_\text{VDL}(\cdot, \cdot, \cdot)$ denotes our sampling strategy and $r \geq 0$ is a hyperparameter. 
The following proposition validates the high cosine similarity between the estimated text sample $\bzth$ and the image feature $\bzi$ that has a similar representation with the unknown text embedding $\bzt$.
\begin{proposition}
\vspace{0.1cm}
Let $\bzth$ be a sample obtained by the proposed sampling strategy $S_\text{VDL}$ defined in \eqref{eq:fake_txt} based on $\bzi$. 
Then, the following inequality always holds for $G(\cdot;\phi^\bzt)$ with arbitrary values of its parameter $\phi^\bzt$:
\begin{align}
\tbzth  \bzi \geq \sqrt{1-r^2}, \nonumber
\end{align}
where $r<1$.
\label{propostion:1}
\end{proposition}
\noindent 
We provide the proof of this proposition in the supplementary document.
Proposition~\ref{propostion:1} implies that a sample $\bzth$ drawn by~\eqref{eq:fake_txt} is guaranteed to be close to $\bzi$ by setting $r$ to a sufficiently small number in $[0,1]$. 
The proposed sampling strategy also results in the acceleration of the optimization procedure by constraining the search space of the variational distribution.
Thanks to the optimization procedure, $\bzth$ is expected to achieve high cosine similarity with $\bzt$, which is also empirically confirmed in Section~\ref{sec:experiments}.

\subsubsection{Generating Samples for Image Embedding}
\label{subsub:image_embedding_reconstruction_samples}
During training, we reconstruct the image embedding samples to compute the second expectation in~\eqref{eq:vi_detail} while we obtain the embeddings to synthesize the images during inference.
Similar to the sampling process of $q(\bzt | \bzi)$, we reconstruct the image embedding based on the text feature $\bzt$  via the sampling process of $p(\bzi | \bzt)$ using a decoder $F(\cdot)$ with $\theta^\bzi$, which is given by   
\begin{align}
\bzih &\sim S_\text{VDL} (\bzt, F, r) \nonumber \\ 
&=\text{Normalize} \left( \bzt + r \cdot \frac{ F(\bzt)}{\| F(\bzt) \|} \right), 
\label{eq:fake_img}
\end{align} 
where $\bzih$ is a reconstructed image embedding. 
According to Proposition~\ref{propostion:1}, $\bzih$ is also expected to have a high cosine similarity with $\bzt$.

\subsubsection{Generating Prior Samples for Text Embedding}
\label{subsub:prior_samples}
For the first stage of training, we require samples from the prior to minimize the difference between the variational and text prior distributions with respect to the parameters in the encoder $G(\cdot)$.
Refer to Section~\ref{sub:training} for the detailed optimization procedure.
We draw a sample $\bzpt$ from the prior distribution using a text corpus, which is available online.
Note that we do not employ an additional dataset consisting of image-text pairs, but randomly choose a sentence $\bxt'$ from the text corpus.
Using the sampled sentence, the representation of the prior sample is obtained by
\begin{align}
\bzpt = f_\text{txt}(\bxt'). 
\end{align}
In our framework, the pretrained CLIP encoder is fixed to reduce the computational burden. 
\begin{figure*}[t]
\centering
\begin{subfigure}[t]{0.43\linewidth}
\centering
\includegraphics[width=0.96\linewidth]{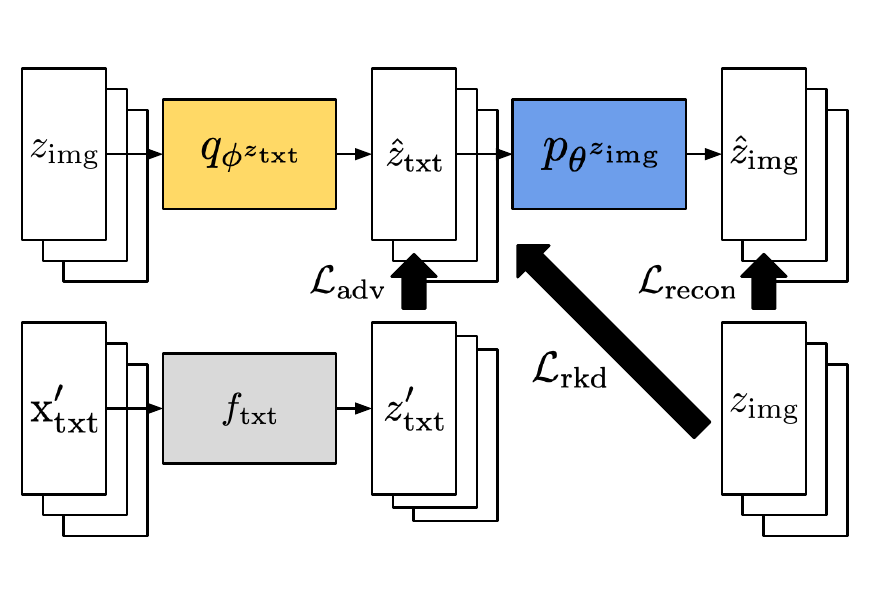}
\caption{First Stage Training}
\label{fig:framework_a}
\end{subfigure}
~ 
\begin{subfigure}[t]{0.26\linewidth}
\centering
\includegraphics[width=0.96\linewidth]{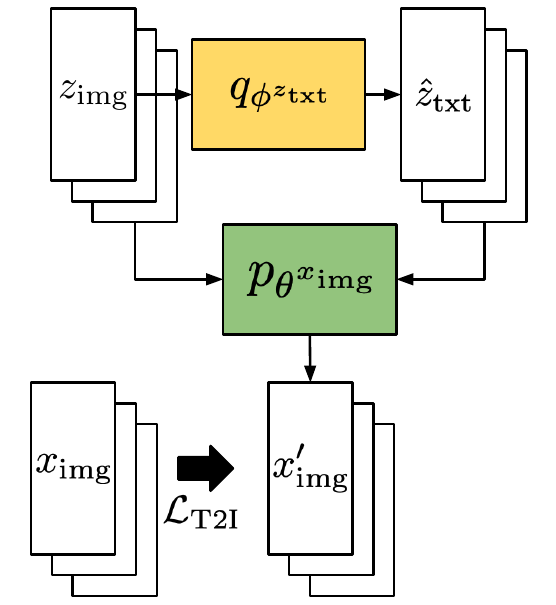}
\caption{Second Stage Training}
\label{fig:framework_b}
\end{subfigure}
~ 
\begin{subfigure}[t]{0.26\linewidth}
\centering
\includegraphics[width=0.96\linewidth]{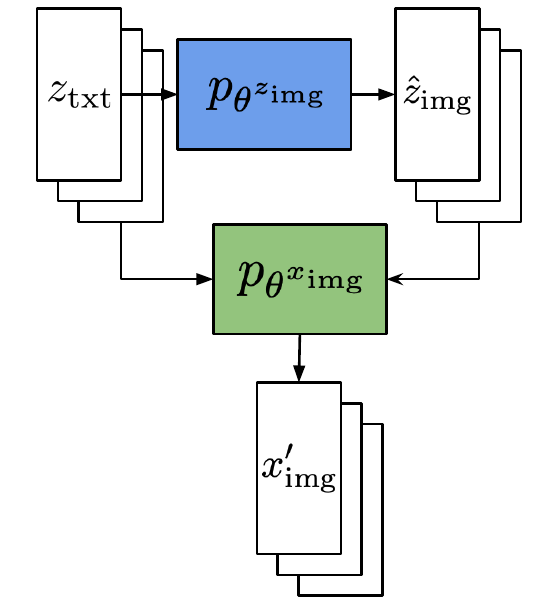}
\caption{Inference}
\label{fig:framework_c}
\end{subfigure}
\caption{Overview of the proposed method. We illustrate the first and second stage training procedures and then depict the inference step.} 
\label{fig:framework}
\end{figure*}

\subsection{Training}
\label{sub:training}
\subsubsection{Robust Objective}
\label{subsub:robust_objective}
The optimization of \eqref{eq:vi_detail} with respect to $ \theta^\bzi $ and $\phi^{\bzt}$ is equivalent to solve the following minimization problem because $\log p_{\theta^{\bxi}}(\bxi | \bzi, \bzt)$ is irrelevant.
Hence, the objective function for the first stage is given by
\begin{align}
\min_{\theta^\bzi, \phi^\bzt} \, & D_\text{KL} (q_{\phi^{\bzt}}( \bzt | \bzi) || p (\bzt) )\nonumber \\ 
 -&\mathbb{E}_{q_{\phi^{\bzt}}( \bzt | \bzi)} [\log p_{\theta^\bzi} (\bzi | \bzt) ].
\label{eq:stage1_objective}
\end{align}
However, the KL divergence is generally intractable unless the two distributions belong to specific families of probability distributions, \eg, Gaussian distribution.  
Such restricted distributions are different from the true distributions of the CLIP features, which only have non-zero densities at the surface of the unit hypersphere in the feature space.
Therefore, the use of the restricted distributions for modeling the variational and prior distributions would cause high approximation errors~\cite{cremer2018inference}.
To reduce the distribution gap, the von Mises-Fisher distribution can be employed but this probability density involves the Bessel function, which also leads to the intractable KL divergence.

To address the issue, we alternatively minimize the difference between the two distributions using the Jensen-Shannon (JS) divergence that also enforces the two distributions to become identical via its minimization.
Contrary to the KL divergence, the JS divergence is always bounded and free from density assumption, which allows us to use implicit generative networks for flexible modeling of the two distributions.
We observe that replacing the KL divergence with the JS divergence is more effective, which will be empirically validated in Section~\ref{subsubsec:JSD}.
By using the property discussed in~\cite{goodfellow2014generative, nowozin2016f}, the objective is reformulated as the following minimax game, which is given by
\begin{align}
\label{eq:stage1_objective_adv}
\min_{\theta^\bzi, \phi^\bzt} \hspace{-.5mm} \max_{\rho^D}  ~~ & \mathbb{E}_{p (\bzt)}[\log D(\bzt)] \nonumber \\
+ & \mathbb{E}_{q_{\phi^{\bzt}}( \bzt | \bzi)} \left[ \log (1 - D(\bzt)) \right] \nonumber \\
- & \mathbb{E}_{q_{\phi^{\bzt}}( \bzt | \bzi)} \left[ \log p_{\theta^{\bzi}}(\bzi | \bzt) \right], 
\end{align}
where $D(\cdot)$ is a discriminator parametrized by $\rho^D$.
The last term of~\eqref{eq:stage1_objective_adv} denoted by $\mathcal{L}_{\text{recon}}$ is expressed as
\begin{align}
\mathcal{L}_{\text{recon}} :&= \mathbb{E}_{q_{\phi^{\bzt}}( \bzt | \bzi)} \left[ - \log p_{\theta^{\bzi}}(\bzi | \bzt) \right] \nonumber \\ 
&=\frac{1}{2\sigma^2} \mathbb{E}_{q_{\phi^{\bzt}}( \bzt | \bzi)} \left[ \| \bzi - \bzih)  \|^2 \right], 
\label{eq:recon}
\end{align}
where we employ the $\ell_2$ loss for reducing the negative log-likelihood with a balancing factor $\sigma$.

\begin{table*}[t!]
\caption{Results of unsupervised text-to-image generation on the MS-COCO~\cite{lin2014microsoft} and Conceptual Captions 3M~\cite{cc3m} datasets using StyleGAN2~\cite{karras2020analyzing}.
    Captioning indicates a text-to-image generation baseline method relying on a state-of-the-art image captioning algorithm~\cite{zhang2021vinvl}, where the results of the baselines are retrieved from~\cite{zhou2022towards}. 
    Methods with asterisks * report the results of our reproduction.
    A bold-faced number denotes the best performance in each column while `--' indicates that the number is unavailable.
    }    
\scalebox{0.90}{
    \begin{tabular}{ccccccc}
        \toprule
      T2I Model & Dataset & Method & IS ($\uparrow$) & FID ($\downarrow$) & $\text{Sim}_\text{txt}$ ($\uparrow$) & $\text{Sim}_\text{img}$ ($\uparrow$)  \\ %
        \midrule
        \multirow{8}{*}{StyleGAN2~\cite{karras2020analyzing}}  & \multirow{4}{*}{MS-COCO~\cite{lin2014microsoft}}  & Captioning~\cite{zhang2021vinvl} & $15.83$ & $56.36$ & - & - \\ %
        & & CLIP-GEN*~\cite{wang2022clip} & $16.94$ & $58.63$ & 0.3042 & -\\ %
        & & $\text{LAFITE}$~\cite{zhou2022towards} & $27.20$ & $18.04$ & 0.0965 &- \\ 
        & & VDL (Ours) & $\textbf{30.30}$ & $\textbf{13.22}$ & \textbf{0.6104} & \textbf{0.7655} \\
        \cmidrule{2-7}
        & \multirow{3}{*}{Conceptual Captions 3M~\cite{cc3m}} &  CLIP-GEN*~\cite{wang2022clip} & $7.88$ & $84.16$ & 0.2896 &-\\ 
        & & $\text{LAFITE}*$~\cite{zhou2022towards} & $16.06$ & $22.95$ & 0.0912 & -\\ 
        & & VDL (Ours) & $\textbf{23.66}$ & $\textbf{17.37}$ & \textbf{0.6237} & \textbf{0.7105} \\ 
         \bottomrule
    \end{tabular}
    }
    \centering
    \label{tab:stylegan2_on_coco}
\end{table*}

\subsubsection{Relational Representation Transfer}
\label{subsub:weak_supervision}
In addition, we encourage the approximate text samples $\bzth$ to mimic the correlation of the observed image embedding samples $\bzi$ in order to mitigate the challenge posed by the lack of supervision. 
The intuition behind this strategy is that the structural relation of text embeddings will resemble that of image representations. 
For example, the two text embeddings should be located close if the image representations are similar, and vice versa.
To impose the constraint on the text embedding samples, we employ a relational knowledge distillation framework~\cite{park2019relational}, which makes a student mimic the relations among data embeddings given by a teacher.
In our framework, we view image embeddings as teacher samples while text embeddings are regarded as student ones.
Therefore, we additionally minimize the relational distillation loss, which is given by
\begin{equation}
\hspace{-0.2cm} \mathcal{L}_\text{rkd} := \mathbb{E}[ \ell_\delta(\psi_A(\bzii, \bzij, \bzik) \hspace{-0.05cm} - \hspace{-0.05cm} \psi_A(\bzthi, \bzthj, \bzthk)) ], 
\label{eq:rkd_angle_loss}
\end{equation}
where the expectation is taken over any triplet image embeddings ($\bzii, \bzij, \bzik$) and the corresponding samples drawn by the variational distribution over ($\bzthi, \bzthj, \bzthk$).
In the above equation, $\ell_\delta(\cdot)$ and $\psi_A(\cdot, \cdot, \cdot)$ are defined as  
\begin{align}
\ell_\delta(a) & :=
\begin{cases}
    \frac{1}{2}a^2,& \text{for } \| a \| \leq \delta, \\
    \delta \cdot (\| a \| - \frac{1}{2} \delta) ,              & \text{otherwise},
\end{cases}
\\
\psi_A(\bz^i, \bz^j, \bz^k) &:= \text{sim}(\bz^i- \bz^j, \bz^i-\bz^k),
\end{align}
where sim$(\cdot, \cdot)$ denotes the cosine similarity between two vectors.
We set the hyperparameter $\delta$ of the Huber loss $\ell_\delta (\cdot)$ to 1 instead of searching for it.

\subsubsection{Total Objective}
\label{subsub:final_objective}
In summary, the final objective function of the first stage is given by
\begin{align}
&\min_{\theta^\bzi, \phi^\bzt}  \max_{\phi^D} \mathcal{L}_{\text{adv}} + \mathcal{L}_{\text{recon}} + \lambda_\text{rkd} \mathcal{L}_{\text{rkd}},
\end{align}
where $\lambda_\text{rkd}$ is a hyperparameter and $\mathcal{L}_\text{adv}$ denotes the first two terms of~\eqref{eq:stage1_objective_adv} constituting the adversarial loss.
Figure~\ref{fig:framework_a} illustrates the optimization procedure of the first stage.
In the case of the semi-supervised setting, we additionally employ the reconstruction loss for labeled examples, $\mathcal{L}_\text{semi}$, which enforces the variational text samples to mimic the true ones.
The reconstruction loss is formally given by  
\begin{equation}
\mathcal{L}_\text{semi} :=   \mathbb{E}_{q_{\phi^{\bzt}}( \bzth | \bzi)} \left[ \| \bzth -\bzt \|_1  \right],
\label{eq:semi}
\end{equation}
where $\| \cdot \|_1$ indicates the $\ell_1$-norm.

%
%
%
\section{Text-to-Image Generative Models}
\label{sub:training_t2i}
After training the first stage of VDL, we leverage the approximate posterior distribution of text features,  $q_{\phi^{\bzt}}( \bzt | \bzi)$, for the unsupervised training of T2I models.
For a T2I generation model with trainable parameters $\theta^{\bxi}$, we maximize \eqref{eq:vi_detail} with respect to $\theta^{\bxi}$, which is equivalent to solve the following problem:
\begin{align}
\max_{\theta^{\bxi}} \mathbb{E}_{q_{\phi^{\bzt}}( \bzt | \bzi)} [ \log p_{\theta^{\bxi}}(\bxi | \bzi, \bzt)], 
\end{align}
where $p_{\theta^{\bxi}}(\bxi | \bzi, \bzt)$ is a T2I generative model. 
Figure~\ref{fig:framework_b} and \ref{fig:framework_c} illustrate the inference procedures for training and testing, respectively.

Contrary to the existing methods such as CLIP-GEN~\cite{wang2022clip} and LAFITE~\cite{zhou2022towards}, which rely only on $\bzt$ for conditional image generation, our approach utilizes $\bzi$ as well as $\bzt$ for conditioning, resulting in better generalization performance thanks to the additional information for the image.
In our framework, $\bzih$ is reconstructed by the learned model, $p_{\theta^{\bzi}}(\bzi | \bzt)$, during inference, and then an image $\bxi$ is finally obtained from the optimized generator, $p_{\theta^{\bxi}}(\bxi | \bzih, \bzt)$.

For a T2I model, we employ StyleGAN2~\cite{karras2020analyzing} while Latent Diffusion Model (LDM)~\cite{LDM} is also adopted, where the experimental results of LDM are provided in the supplementary material. 
When the StyleGAN2 synthesizes images conditioned on text, we replace each style vector $\mathbf{s}^u$ with its conditional counterpart $\mathbf{s}^c$ following~\cite{zhou2022towards}, which is formally given by
\begin{equation}
\mathbf{s}^c = h([\mathbf{s}^u ; g([\bzt;\bzi])]),
\end{equation}
where $h(\cdot)$ denotes an affine transform, $g(\cdot)$ is a neural network with two fully-connected layers, and $[\cdot;\cdot]$ is the concatenation operator.

%

\section{Experiments}
\label{sec:experiments}
This section compares the proposed method referred to as VDL with existing approaches on the standard datasets under unsupervised and semi-supervised text-to-image generation settings, and analyzes the proposed components.
\begin{table*}[t!]
\caption{Results of semi-supervised text-to-image generation on the MS-COCO dataset using StyleGAN2.
    The `Ratio' column shows the fractions of the labeled text captions in each dataset. 
    }
\scalebox{0.90}{
    \begin{tabular}{cccccccc}
        \toprule
        Model & Dataset & Method & Ratio & IS ($\uparrow$) & FID ($\downarrow$) & $\text{Sim}_\text{txt}$ ($\uparrow$) & $\text{Sim}_\text{img}$ ($\uparrow$) \\
        \midrule
        \multirow{7}{*}{StyleGAN2~\cite{karras2020analyzing}} & \multirow{7}{*}{MS-COCO~\cite{lin2014microsoft}} & LAFITE~\cite{zhou2022towards}  & $0.0$ & $27.20$ & $18.04$ & 0.0965 & -- \\
        & & LAFITE*~\cite{zhou2022towards}  & $0.1$ & $18.82$ & $20.65$ & 0.8340 & -- \\
        & & LAFITE*~\cite{zhou2022towards}  & $0.2$ & $21.19$ & $17.74$ & 0.8373 & -- \\
      & & LAFITE*~\cite{zhou2022towards}  & $0.3$ & $21.39$ & $15.76$ & \textbf{0.8385} & -- \\
      & & VDL (Ours) & $0.0$ & $30.30$ & $13.22$ & 0.6237 & 0.7105 \\
      &  & VDL (Ours) & $0.1$ & $32.81$ & $\textbf{11.24}$ & 0.7130 & 0.7536\\
      &  & VDL (Ours) & $0.2$ & $\textbf{33.90}$ & $\textbf{11.24}$ & 0.7261 & \textbf{0.7596}\\
        \bottomrule
    \end{tabular}
    }
    \centering
    \label{tab:Semi}
\end{table*}

\subsection{Datasets}
\label{subsec:datasets}
We employ MS-COCO~\cite{lin2014microsoft} and Conceptual Captions 3M~\cite{cc3m} (CC3M) datasets, which are widely used for the evaluation of text-to-image generation tasks. 
As training and validation datasets, the MS-COCO dataset contains 82k and 40k images while the CC3M dataset consists of 3.3M and 16k examples, respectively.   
As a preprocessing, we resize all images to 256$\times$256 pixels for training a T2I network while the images are resized to 224$\times$224 pixels before feeding them into the CLIP image encoder.   
As text corpora, we use text captions in CC3M for MS-COCO and 3 million randomly sampled texts from Conceptual 12M~\cite{cc12m} for CC3M.
Note that we do not utilize the image-text pairs for the proposed algorithm like other unsupervised methods such as CLIP-GEN~\cite{wang2022clip} and LAFITE~\cite{zhou2022towards}.

\subsection{Evaluation Metrics}
\label{subsec:evaluation_metrics}
We select the Fr\'{e}chet Inception Distance (FID)~\cite{heusel2017gans} and Inception Score (IS)~\cite{salimans2016improved} to evaluate and compare the visual quality of generated images. 
We measure the two metrics following the experimental protocol of previous works~\cite{zhu2019dm, dalle, zhou2022towards} for fair comparisons.
Additionally, we report $\text{Sim}_\text{txt}$ and $\text{Sim}_\text{img}$ using the validation dataset, which are calculated by the expected cosine similarity between the true and predicted text features and the similarity between the image and its inferred embeddings, respectively.
Note that $\text{Sim}_\text{img}$ cannot be measured for the baseline algorithms since they do not consider reconstructing the image feature. 
The discrepancy between the training and inference becomes lower when each of the similarities is higher.

\subsection{Implementation Details}
\label{subsec:implementation_details}
The proposed method is implemented with the official code of LAFITE\footnote{https://github.com/drboog/Lafite} based on PyTorch~\cite{paszke2019pytorch}. 
For the encoder $G(\cdot)$, decoder $F(\cdot)$, and discriminator $D(\cdot)$, we adopt networks consisting of 10 fully-connected layers with the leaky ReLU activations, where each hidden layer has 2048 units.
For the discriminator, we add an $R_1$ regularization~\cite{mescheder2018training} for training stability, which suppresses the magnitudes of gradients. 
We use the Adam optimizer~\cite{diederik2015adam} for the three networks with an initial learning rate of $0.001$ with a batch size of 512.

During the second training stage, we train a T2I model using StyleGAN2~\cite{karras2020analyzing} to follow the experimental protocol used in LAFITE~\cite{zhou2022towards} for fair comparisons.
Specifically, we optimize StyleGAN2 using the Adam optimizer, and set a batch size to 64 and an initial learning rate to $2.5 \times 10^{-3}$. 
Also, $R_1$ regularization is also performed for the discriminator every 16 iterations to save training time.

\subsection{Unsupervised Setting Results}
\label{subsec:unsupervised}
Table~\ref{tab:stylegan2_on_coco}  shows the text-to-image generation results on the MS-COCO~\cite{lin2014microsoft} and CC3M~\cite{cc3m} datasets under the unsupervised setting, where the image captions are not available during training.
As demonstrated in the table, VDL achieves the best performance in terms of $\text{Sim}_\text{txt}$, FID, and IS by large margins both on the two datasets.
Although the noise injection for text prediction in LAFITE~\cite{zhou2022towards} degrades $\text{Sim}_\text{txt}$, LAFITE achieves better T2I performance in terms of FID and IS than CLIP-GEN. 
This is partly because the regions corresponding to the noisy text features predicted by LAFITE are larger than the deterministic point features given by CLIP-GEN and LAFITE consequently has more chance to identify accurate text features within the region.

\subsection{Semi-Supervised Setting Results}
\label{subsec:semi-supervised}
In Table~\ref{tab:Semi}, we present the performance of LAFITE~\cite{zhou2022towards} and VDL in semi-supervised learning scenarios on the MS-COCO dataset~\cite{lin2014microsoft}, where only a fraction of image and text pairs are accessible; the ratio of $0.0$ indicates the unsupervised setting.
According to the table, the proposed method outperforms LAFITE by large margins in terms of FID and IS.
Also, VDL trained even in the unsupervised setting outperforms LAFITE with the ratio of $0.3$, which implies that the proposed method is more annotation-efficient.
As more text captions become available, VDL obtains higher IS, $\text{Sim}_\text{txt}$, and $\text{Sim}_\text{img}$ while FID is unfortunately saturated at an early stage.
In terms of $\text{Sim}_\text{txt}$ under semi-supervised settings, LAFITE outperforms VDL. 
That this is partly because LAFITE trains a learnable network on the paired data by making the predicted text embeddings close to the true ones.
However, the strategy turns out to be ineffective for improving FID and even degrades the IS score.
\begin{figure*}[t]
\centering
\includegraphics[width=\linewidth]{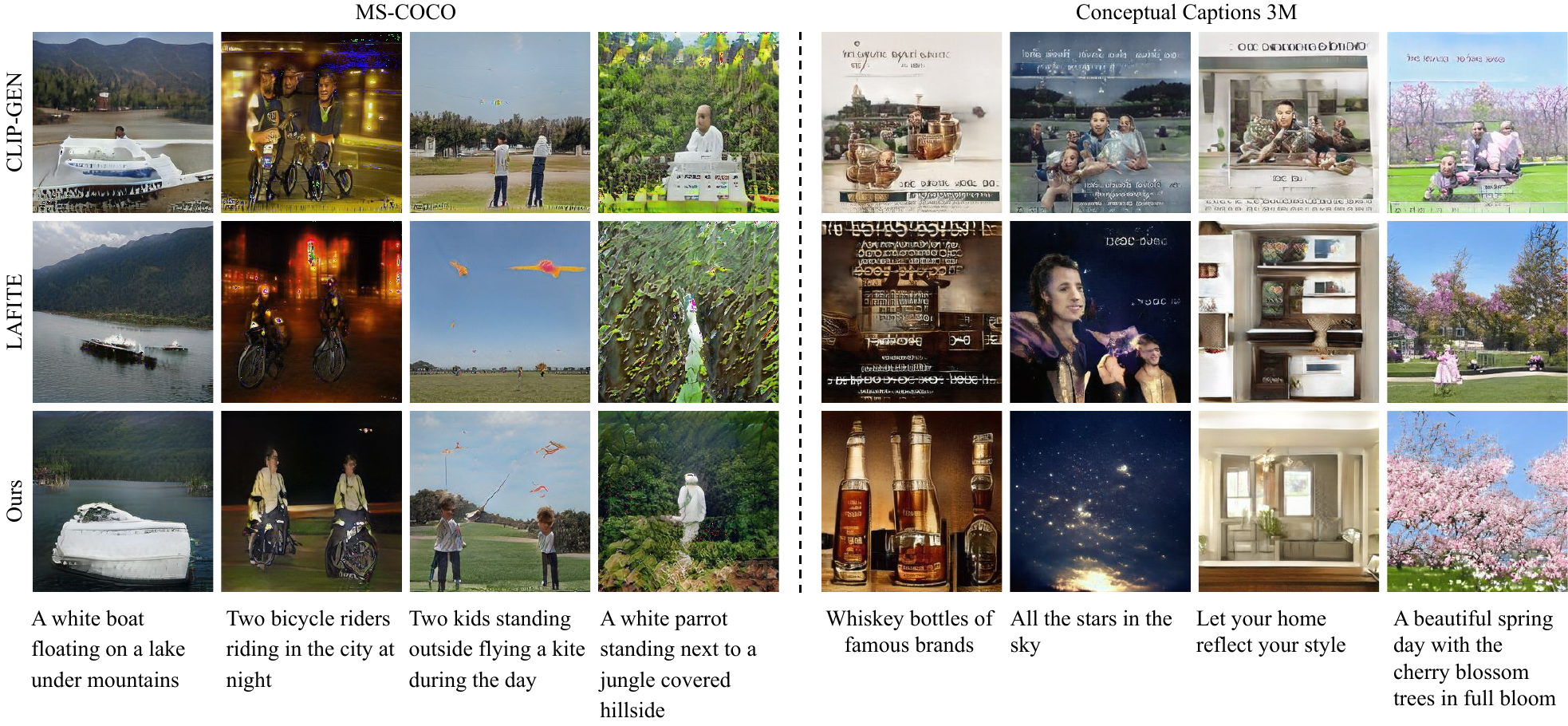}
\caption{Qualitative results on the MS-COCO and Conceptual Captions 3M datasets using StyleGAN2. VDL generates visually higher-quality images than LAFITE and CLIP-GEN.}
\label{fig:coco_stylegan2_qual}
\end{figure*}
\begin{table}[t!]
    \caption{Comparison between the KL divergence with the JS divergence under the unsupervised setting on the MS-COCO dataset using StyleGAN2. 
    KLD optimizes~\eqref{eq:stage1_objective} with the assumption that the variational distribution follows the Gaussian distribution with Gaussian mixture models for the prior.
    DualKLD optimizes the dual representation of the KL objective in~\cite{donsker1983asymptotic}.
    }
\scalebox{0.90}{
    \begin{tabular}{lccccc}
        \toprule
         Method & IS ($\uparrow$) & FID ($\downarrow$) & $\text{Sim}_\text{txt}$ ($\uparrow$) & $\text{Sim}_\text{img}$ ($\uparrow$) \\
        \midrule
        KLD & $20.89$ & $34.42$ & 0.1396 & 0.2833 \\
        DualKLD & $27.83$ & $15.60$ & \textbf{0.6236} & 0.7440\\
        JSD (VDL) & $\textbf{30.30}$ & $\textbf{13.22}$ & 0.6104 & \textbf{0.7655} \\ 
        \bottomrule
    \end{tabular}
    }
    \centering
    \label{tab:KLvsJSD}
\end{table}
\subsection{Analysis}
\label{subsec:ablation}

\subsubsection{Jensen-Shannon Divergence}
\label{subsubsec:JSD}
We study the effect of using the JS divergence instead of KL divergence under the unsupervised setting with the StyleGAN2 architecture on MS-COCO. 
As reported in Table~\ref{tab:KLvsJSD}, our strategy using the JS divergence is more effective than employing the KL divergence and its variation. 
Specifically, we compute the objective in~\eqref{eq:stage1_objective} using two different ways. 
First, motivated by variational autoencoders~\cite{kingmaauto}, we model the variational and prior distributions with a Gaussian distribution and its mixture, respectively, and this approach is referred to as KLD.
The other method denoted by DualKLD replaces the KL divergence with its dual form, the Donsker-Varadhan representation~\cite{donsker1983asymptotic}, which also performs the minimax optimization.
In the case of KLD, the performance significantly degrades because the Gaussian assumptions for the variational and prior distributions are not effective; the Gaussian distribution has a non-zero density outside the unit-hypersphere, where CLIP text features are not located.
On the other hand, DualKLD, which is free from the distribution restriction, outperforms KLD although it is still worse than our approach.
However, we observe that DualKLD is sensitive to the hyperparameter partly due to its unbounded property contrary to the JS divergence.

\begin{table}[t!]
\caption{Ablation study results on MS-COCO with StyleGAN2 under an unsupervised setting. 
     VDL w/o $S_\text{VDL}$ directly predicts a text feature and then normalizes it to locate at the unit-hypersphere without using $S_\text{VDL}$ while VDL w/o $\mathcal{L}_\text{rkd}$ does not employ $\mathcal{L}_\text{rkd}$. 
    }
\scalebox{0.90}{
    \begin{tabular}{lccccc}
        \toprule
         Method & IS ($\uparrow$) & FID ($\downarrow$) & $\text{Sim}_\text{txt}$ ($\uparrow$) & $\text{Sim}_\text{img}$ ($\uparrow$) \\
        \midrule
        VDL w/o $S_\text{VDL}$ & $19.86$ & $36.37$ & 0.5498 & 0.5256 \\
        VDL w/o $\mathcal{L}_\text{rkd}$ & $28.45$ & $15.75$ & \textbf{0.6128} & \textbf{0.7637} \\
        VDL (Ours) & $\textbf{30.30}$ & $\textbf{13.22}$ & \textbf{0.6104} & \textbf{0.7655} \\ 
        \bottomrule
    \end{tabular}
    }
    \centering
    \label{tab:ablation}
\end{table}

\subsubsection{Component analysis}
We analyze the contributions of the individual components in our approach.
As presented in Table~\ref{tab:ablation}, $S_\text{VDL}$ is helpful for improving $\text{Sim}_\text{txt}$, $\text{Sim}_\text{img}$, and T2I performance in unsupervised settings.
Although $\mathcal{L}_\text{rkd}$ is conceptually irrelevant to reduce the discrepancy between the true and predicted text features, it improves generation performance by learning relational embeddings between two modalities.

\subsubsection{Qualitative Results}
Figure~\ref{fig:coco_stylegan2_qual} visualizes generation results on the MS-COCO and CC3M datasets using CLIP-GEN~\cite{wang2022clip}, LAFITE~\cite{zhou2022towards}, and VDL. 
As illustrated in the figure, the proposed method successfully generates images based on given sentences with enhanced visual quality while the others sometimes fail to understand the overall meaning of text captions.

%


\section{Conclusion}
\label{sec:conclusion}
We presented an annotation-efficient method for text-to-image generation when image and text caption pairs are rarely available or text information is completely inaccessible.  
To address the challenge, we employ the off-the-shelf CLIP model to estimate hidden text features given observable images, where we rely on the CLIP's multi-modal joint embedding quality.
To further improve the quality of text embedding, we approximate its intractable true posterior probability by exploiting the variational inference technique.
Given the inferred features and their image embeddings, we learn a conditional generative model to reconstruct images. 
Experimental results verify that the proposed method achieves outstanding performance on the unsupervised and semi-supervised learning environments.

\paragraph{Acknowledgments}
This work was partly supported by the Institute of Information \& communications Technology Planning \& Evaluation (IITP) grants funded by the Korea government (MSIT) [No.2022-0-00959, (Part 2) Few-Shot Learning of Causal Inference in Vision and Language for Decision Making; No.2021-0-01343, Artificial Intelligence Graduate School Program (Seoul National University); No.2021-0-02068,AI Innovation Hub].

{\small
\bibliographystyle{ieee_fullname}
\bibliography{egbib}
}
\setcounter{proposition}{0}
\appendix
\section{Appendix}
\label{sec:appendix} 
This document first provides the proof of Proposition 1. 
Then, we present additional results on the LN-COCO~\cite{pont2020connecting} dataset under the unsupervised T2I generation task based on StyleGAN2~\cite{karras2020analyzing}. 
We also evaluate the performance of the T2I results using a diffusion-based text-to-image generative model~\cite{LDM} to validate the generality of the proposed approach.
Finally, we demonstrate additional qualitative results supplementing Figure 3 of the main paper, which visualizes the synthesized results on the MS-COCO~\cite{lin2014microsoft} and Conceptual Captions 3M~\cite{cc3m} datasets given by CLIP-GEN~\cite{wang2022clip}, LAFITE~\cite{zhou2022towards}, and VDL based on StyleGAN2 under the unsupervised setting.

\subsection{Proof of Proposition 1}
\label{sec:Proof}
\begin{proposition}
Let $\bzth$ be a sample obtained by the proposed sampling strategy $S_\text{VDL}$ defined in \eqref{eq:fake_txt} based on $\bzi$. 
Then, the following inequality always holds for $G(\cdot;\phi^\bzt)$ with arbitrary values of its parameter $\phi^\bzt$:
\begin{equation}
\tbzth  \bzi \geq \sqrt{1-r^2}, \nonumber
\end{equation}
where $r<1$.
\label{propostion:1}
\end{proposition}
\begin{proof}
The inner product can be expressed as 
\begin{equation}
\tbzth  \bzi = \tbzi  \frac{\bzi + r \cdot \bgi}{\| \bzi + r \cdot \bgi \|},  
\label{eq:prop_LHS}
\end{equation}
where $\bgi$ is equivalent to $\text{Normalize}(G(\bzi))$.
In addition, the denominator in \eqref{eq:prop_LHS} is given by
\begin{align}
\| \bzi + r \cdot \bgi \| &= \sqrt{(\bzi + r \cdot \bgi)^T (\bzi + r \cdot \bgi)} \nonumber \\
&= \sqrt{1 + 2 r \cdot \tbzi \bgi + r^2}.
\label{eq:magnitude_variational_samples}
\end{align}
Based on the two equations, we have 
\begin{align}
\tbzth  \bzi &=  \frac{\tbzi (\bzi + r \cdot \bgi)}{ \sqrt{1 + 2 r \cdot \tbzi \bgi + r^2}} \nonumber \\
&=  \frac{1 + r \cdot \tbzi \bgi}{ \sqrt{1 + 2 r \cdot \tbzi \bgi + r^2}} \nonumber \\
&=\frac{(1 + 2 r \cdot \tbzi \bgi + r^2) + (1 - r^2)}{ 2\sqrt{1 + 2 r \cdot \tbzi \bgi + r^2}} \nonumber \\
&\geq \sqrt{1-r^2}, \nonumber
\end{align}
where the last inequality is derived by using the inequality of arithmetic and geometric means.
\end{proof}

\subsection{Experiments on LN-COCO}
\label{sec:Unsupervised}
We compare the proposed method with CLIP-GEN~\cite{wang2022clip} and LAFITE~\cite{zhou2022towards} on the LN-COCO~\cite{pont2020connecting} dataset using StyleGAN2 under the unsupervised setting. 
As presented in Table~\ref{tab:LN-COCO}, our method outperforms existing approaches by large margins.
Figure~\ref{fig:mscoco} depicts synthesized images given by VDL and CLIP-GEN~\cite{wang2022clip}, where the results of LAFITE are not provided since its pre-trained model is not publicly available. 

\subsection{Ablation on T2I models}
\label{sec:LDM}
We remark that our approach is model-agnostic; any type of conditional image generation network is applicable to our framework.
To validate the effectiveness of the proposed method without carefully designing the T2I network, we replace StyleGAN2~\cite{karras2020analyzing} with a diffusion model, Latent Diffusion Model (LDM)~\cite{LDM}, and perform experiments on the MS-COCO~\cite{lin2014microsoft} and Conceptual Captions 3M~\cite{cc3m} (CC3M) datasets under the unsupervised setting.

\subsubsection{Implementation Details}
For the second-stage training, we optimize LDM~\cite{LDM} for 150k and 300k iterations on the MS-COCO~\cite{lin2014microsoft} and Conceptual Captions 3M~\cite{cc3m} datasets using the Adam optimizer with a batch size of 64 and an initial learning rate of $6.4 \times 10^{-5}$.
We set the resolution of the latent space to 64, where a pretrained Vector Quantized GAN~\cite{VQGAN} is selected as a latent perceptual compression network without an extra fine-tuning.
Following the latent conditioning strategy used in~\cite{preechakul2022diffusion}, we inject CLIP features to the noisy predictions of the backbone network based on U-Net~\cite{ronneberger2015u} inside the LDM framework~\cite{LDM}.
Specifically, we replace the last group normalization layer in each residual block of the U-Net with an adaptive group normalization layer whose scale and shift parameters are computed by applying a single fully connected layer to the temporal positional embeddings.
For conditioning given sentences, the outputs of the normalization layer are further multiplied with the projected CLIP features using a single fully connected layer.

\subsubsection{Results unser Unsupervised Setting}
We present quantitative results in Table~\ref{tab:ldm} while the generated images on the MS-COCO~\cite{lin2014microsoft} and Conceptual Captions 3M~\cite{cc3m} datasets are provided in Figure~\ref{fig:mscoco_ldm} and \ref{fig:cc3m_ldm}, respectively.
These results show that VDL archives superior performance when combined with the diffusion based LDM~\cite{LDM} for the T2I model, and the proposed method is agnostic to the types of the T2I network.
%

\subsection{Additional Qualitative Results}
\label{sec:additional_qualitative}
Figure~\ref{fig:mscoco} and \ref{fig:cc3m} visualize additional qualitative results from the proposed approach compared to existing methods.
The results clearly show that VDL generates visually more faithful and realistic images considering the given text descriptions and the natural image distribution while the other two methods often fail to meet text conditions and/or generate natural images.
\begin{table*}[t!]
    \caption{Results of unsupervised text-to-image generation on the LN-COCO~\cite{pont2020connecting} dataset using StyleGAN2~\cite{karras2020analyzing}.
    Methods with asterisks (*) report the results of our reproduction. 
    A bold-faced number denotes the best performance in each column. 
    }
\scalebox{0.90}{
    \begin{tabular}{lccccccc}
        \toprule
         T2I Model & Dataset & Method & IS ($\uparrow$) & FID ($\downarrow$) & $\text{Sim}_\text{txt}$ ($\uparrow$) & $\text{Sim}_\text{img}$ ($\uparrow$) \\
        \midrule
       \multirow{3}{*}{StyleGAN2~\cite{karras2020analyzing} } & \multirow{3}{*}{LN-COCO~\cite{pont2020connecting}} &CLIP-GEN*~\cite{wang2022clip} & $12.12$ & $83.87$ & $0.2750$ & $-$ \\
       & &  LAFITE~\cite{zhou2022towards} & $18.49$ & $38.95$ & $0.0872$ & $-$ \\
       & &  VDL (Ours) & $\textbf{21.55}$ & $\textbf{31.33}$ & $\textbf{0.6118}$ & $\textbf{0.7025}$ \\ 
        \bottomrule
    \end{tabular}
    }
    \centering
    \label{tab:LN-COCO}
\end{table*}
\begin{figure*}[!t]
\centering
\includegraphics[width=\linewidth]{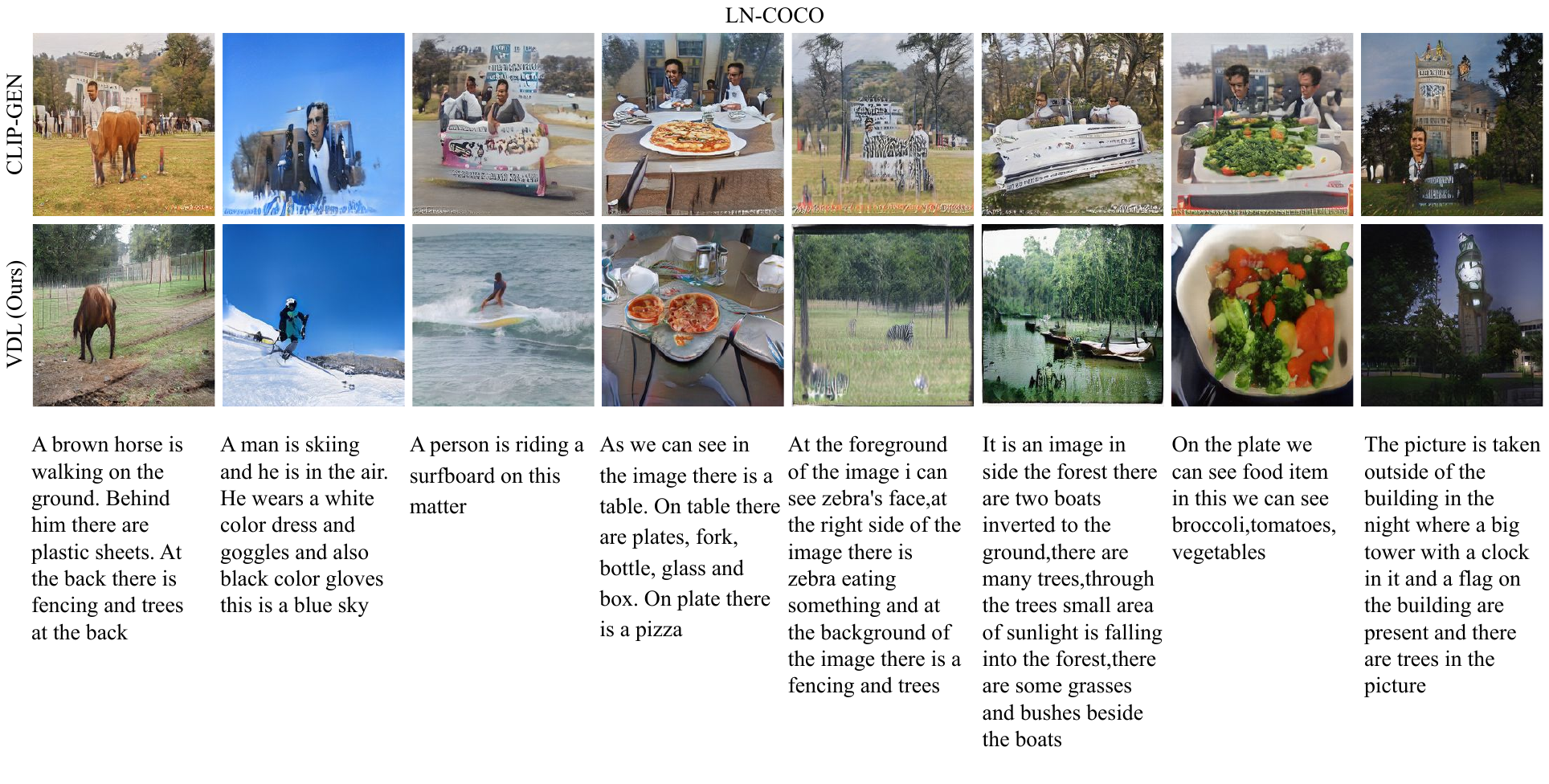}
\caption{Qualitative results on the LN-COCO dataset using StyleGAN2. VDL generates visually higher-quality images than CLIP-GEN.}
\label{fig:mscoco}
\end{figure*}
\begin{table*}[t!]
\caption{Results of unsupervised text-to-image generation on the MS-COCO~\cite{lin2014microsoft} and Conceptual Captions 3M~\cite{cc3m} datasets using LDM~\cite{LDM}. 
    }    
\scalebox{0.90}{
    \begin{tabular}{ccccccc}
        \toprule
      T2I Model & Dataset & Method & IS ($\uparrow$) & FID ($\downarrow$) & $\text{Sim}_\text{txt}$ ($\uparrow$) & $\text{Sim}_\text{img}$ ($\uparrow$)  \\ 
        \midrule
        \multirow{6}{*}{LDM~\cite{LDM}}  & \multirow{3}{*}{MS-COCO~\cite{lin2014microsoft}}  & CLIP-GEN*~\cite{wang2022clip} & $12.96$ & $48.14$ & 0.3042 & -\\ 
        & & LAFITE*~\cite{zhou2022towards} & $16.53$ & $23.92$ & 0.0965 &- \\ 
        & & VDL (Ours) & $\textbf{23.25}$ & $\textbf{13.68}$ & \textbf{0.6104} & \textbf{0.7655} \\ 
        \cmidrule{2-7}
        & \multirow{3}{*}{Conceptual Captions 3M~\cite{cc3m}} &  CLIP-GEN*~\cite{wang2022clip} & $10.08$ & $47.53$ & 0.2896 &-\\ 
        & & LAFITE*~\cite{zhou2022towards} & $10.98$ & $33.98$ & 0.0912 & -\\ 
        & & VDL (Ours) & $\textbf{15.09}$ & $\textbf{23.03}$ & \textbf{0.6237} & \textbf{0.7105} \\ 
         \bottomrule
    \end{tabular}
    }
    \centering
    \label{tab:ldm}
\end{table*}
\begin{figure*}[!t]
\centering
\includegraphics[width=\linewidth]{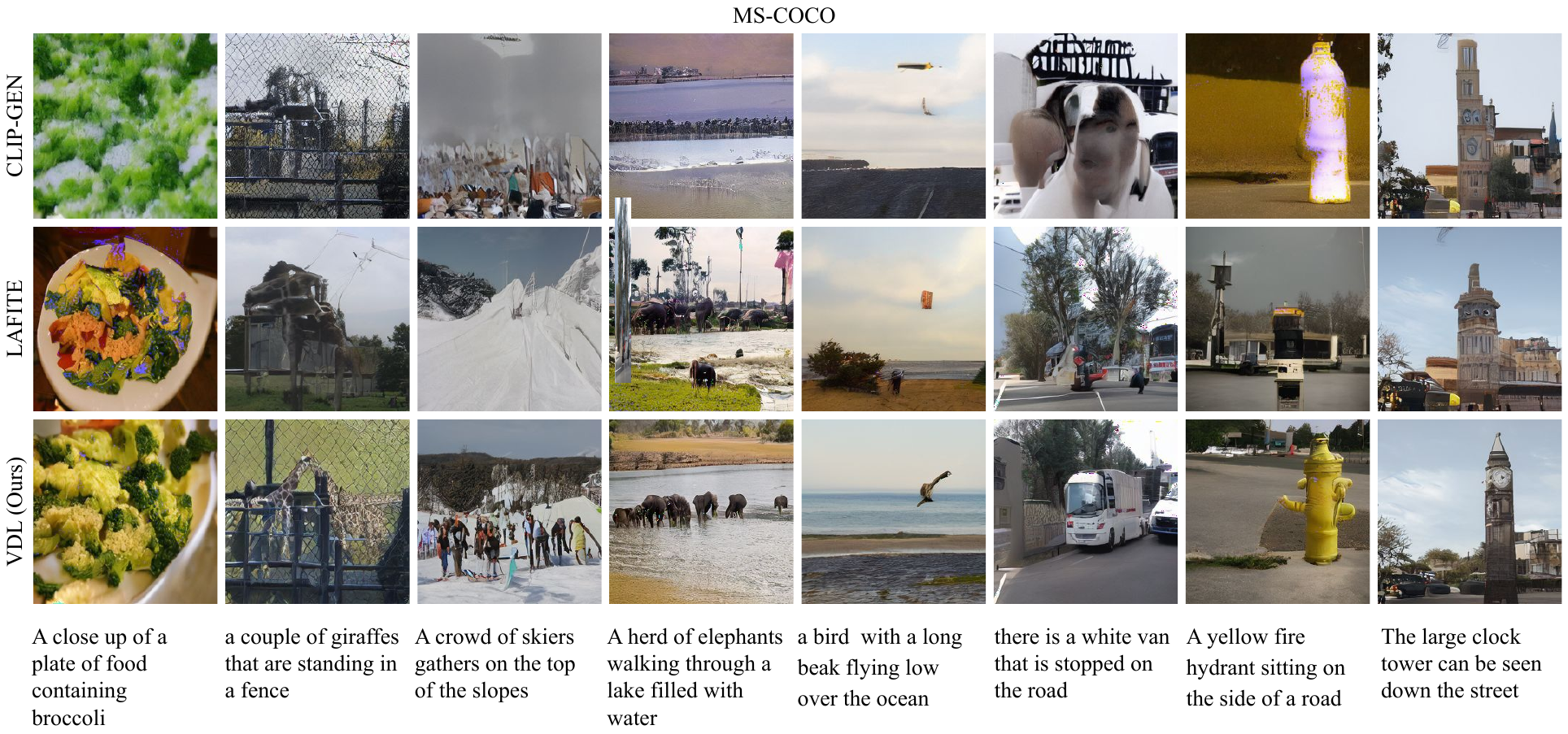}
\caption{Qualitative results on the MS-COCO dataset using LDM. VDL generates visually higher-quality images than LAFITE and CLIP-GEN.}
\label{fig:mscoco_ldm}
\end{figure*}
\begin{figure*}[!t]
\centering
\includegraphics[width=\linewidth]{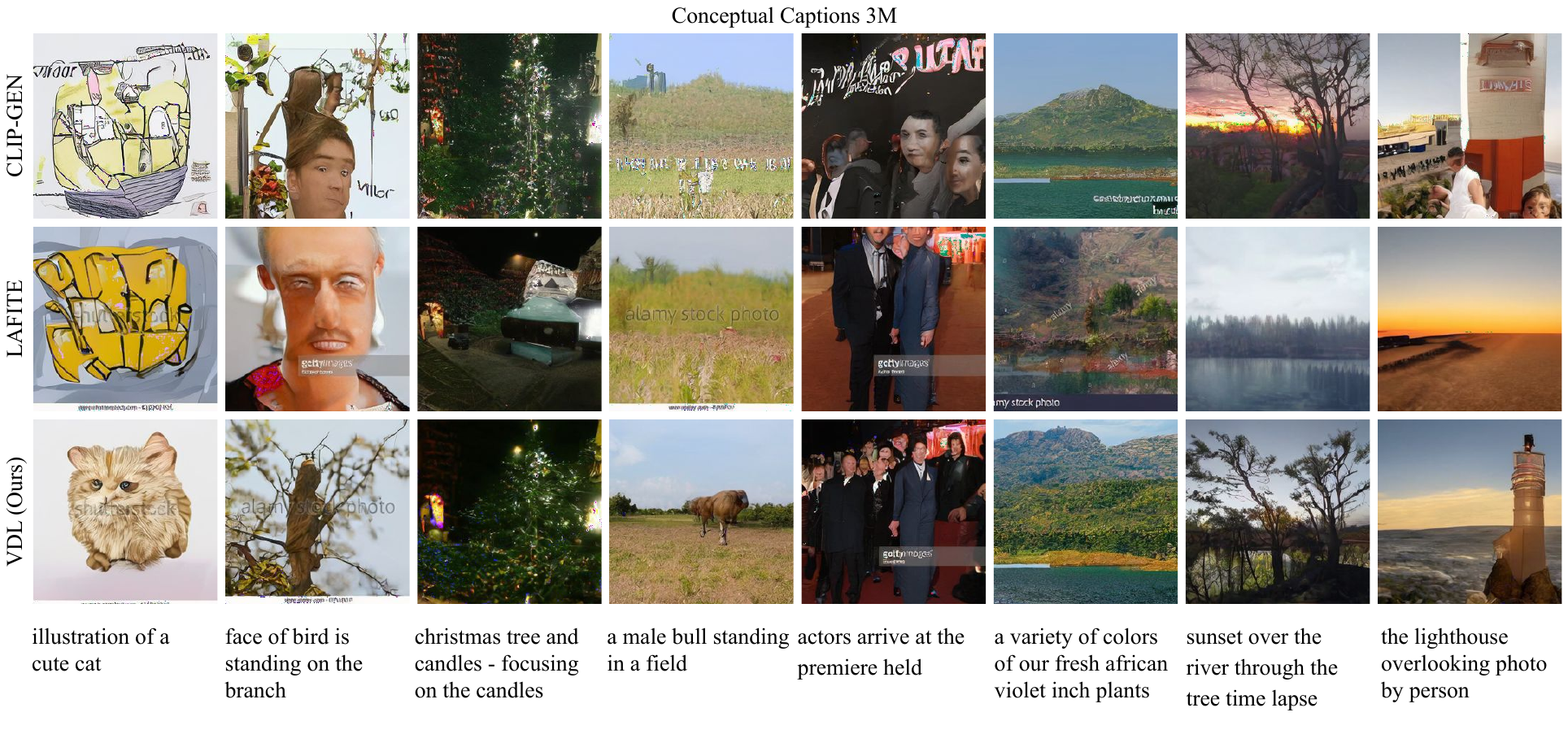}
\caption{Qualitative results on the Conceptual Captions 3M dataset using LDM. VDL generates visually higher-quality images than LAFITE and CLIP-GEN.}
\label{fig:cc3m_ldm}
\end{figure*}
\begin{figure*}[!t]
\centering
\includegraphics[width=\linewidth]{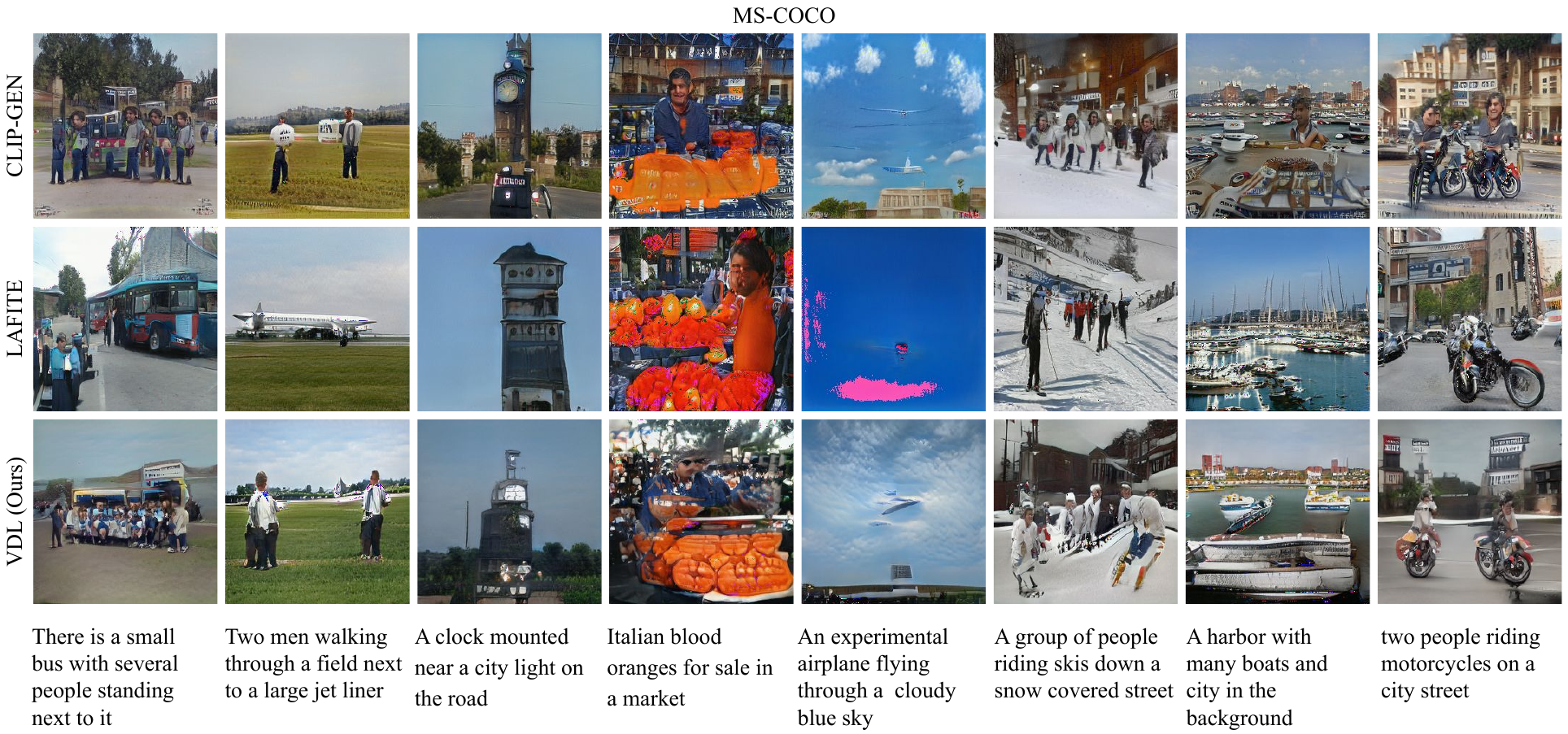}
\caption{Additional qualitative results on the MS-COCO dataset using StyleGAN2. VDL generates visually higher-quality images than LAFITE and CLIP-GEN.}
\label{fig:mscoco}
\end{figure*}
\begin{figure*}[!t]
\centering
\includegraphics[width=\linewidth]{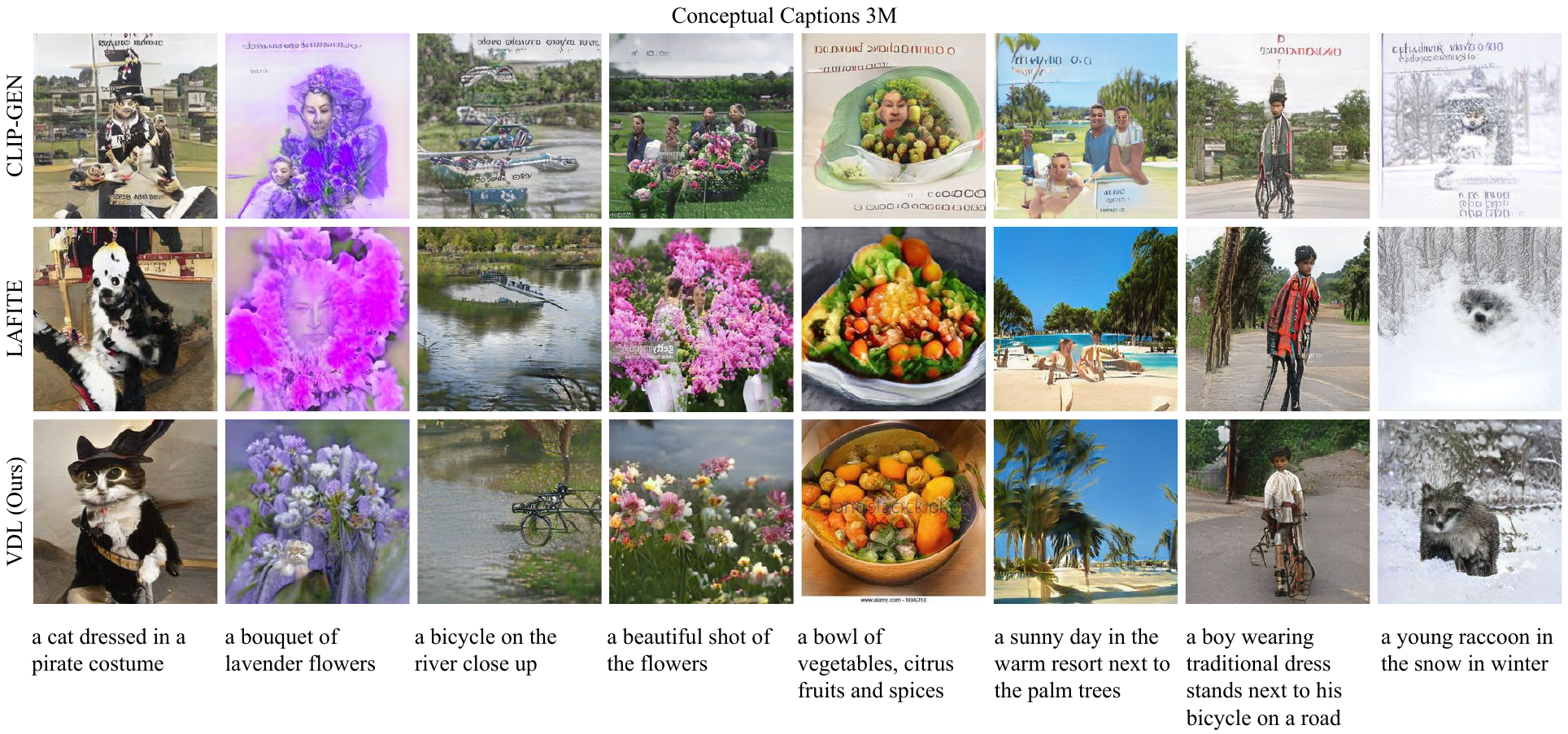}
\caption{Additional qualitative results on the Conceptual Captions 3M dataset using StyleGAN2. VDL generates visually higher-quality images than LAFITE and CLIP-GEN.}
\label{fig:cc3m}
\end{figure*}

\end{document}